\documentclass[letterpaper]{article}
\usepackage{ijcai18}
\usepackage{times}
\usepackage{helvet}
\usepackage{courier}
\usepackage{url}
\usepackage{graphicx}
\usepackage{multibib}

\newcites{add}{Additional References}

\renewcommand{\paragraph}[1]{\noindent\textbf{#1}}

\usepackage{microtype}
\usepackage{latexsym} 
\usepackage{theorem}
\usepackage{xspace}
\usepackage{amsmath}
\usepackage{mathtools}
\usepackage{amssymb}
\usepackage{booktabs}
\usepackage{paralist}
\usepackage{nicefrac}
\usepackage{tikz}
\usepackage{todonotes}
\usepackage{etoolbox}
\usepackage{stmaryrd}

\usepackage[ruled,vlined]{algorithm2e}

\mathchardef\mhyp="2D

\newtheorem{lemma}{Lemma}

\newtheorem{definition}{Definition}
\newtheorem{theorem}{Theorem}
\newtheorem{remark}{Remark}

\newtheorem{example}{Example}

\newenvironment{proof}{\noindent{\textit{Proof.}}}{\hfill$\square$\break}
\newenvironment{sproof}{\noindent{\textit{Proof 
(sketch).}}}{\hfill$\square$\break}
\newcommand{\varthr}{\mathit{thr}}
\newcommand{\varrboundr}{\mathit{rbound}}
\newcommand{\gpred}{\hat{g}}
\newcommand{\varh}{f}

\newcommand{\pomdp}{P}

\newcommand{\states}{S}

\newcommand{\act}{\mathcal{A}}
\newcommand{\trans}{\delta}
\newcommand{\obs}{\mathcal{Z}}
\newcommand{\reward}{r}
\newcommand{\obsmap}{\mathcal{O}}
\newcommand{\distr}{\mathcal{D}}
\newcommand{\initd}{\lambda}

\newcommand{\supp}{\mathrm{Supp}}

\newcommand{\reals}{\mathbb{R}}

\newcommand{\belief}{b}
\newcommand{\discsum}{\mathsf{Disc}}
\newcommand{\discount}{\gamma}

\newcommand{\obsfunc}{\obsmap}
\newcommand{\hist}{h}
\newcommand{\ob}{o}

\newcommand{\histfunc}{H}

\newcommand{\E}{\mathbb{E}}
\newcommand{\evalue}{\mathit{eVal}}

\newcommand{\thr}{\tau}
\newcommand{\Rset}{\mathbb{R}}
\newcommand{\Nset}{\mathbb{N}}

\newcommand{\len}[1]{\mathit{len}(#1)}
\newcommand{\belseq}[2]{#1_{#2}}
\newcommand{\discpath}[2]{\def\EmptyTest{#1}\ifdefempty{\EmptyTest}{\discsum_{#2}}{\discsum_{#2}(#1)}}

\newcommand{\risk}{\mathit{rl}}
\newcommand{\probm}{\mathbb{P}}
\newcommand{\rbound}{\alpha}
\newcommand{\eps}{\varepsilon}

\newcommand{\riskvalue}[2]{\mathit{rVal}(#1,#2)}

\newcommand{\accpayoff}[1]{\mathsf{Payoff}_{#1}}

\newcommand{\node}{n}
\newcommand{\tree}{\mathcal{T}}
\newcommand{\econst}{K}
\newcommand{\rewards}{\mathit{rew}}
\newcommand{\timeout}{\mathit{timeout}}

\newcommand{\trsafe}{\mathcal{T}_{\mathit{safe}}}
\newcommand{\trexpl}{\mathcal{T}_{\mathit{\expl}}}
\newcommand{\trsearch}{\mathcal{T}_{\mathit{\search}}}
\newcommand{\hroot}{h_{\mathit{root}}}
\newcommand{\treeclos}{\hat{\tree}_{\expl}}

\usepackage{nicefrac}
\usepackage{tikz}
\usetikzlibrary{fit}
\usetikzlibrary{backgrounds}
\usetikzlibrary{automata}
\usetikzlibrary{shapes}
\usetikzlibrary{matrix}
\usetikzlibrary{fit}
\usetikzlibrary{calc}
\usetikzlibrary{positioning}
\usetikzlibrary{intersections}

\tikzstyle{Player1}=[circle, thick, minimum size=0.6cm, inner 
sep=0cm,draw=black]
\tikzstyle{State}=[circle, font = \small, thick, minimum size=0.6cm, inner 
sep=0cm,draw=black,fill=white]
\tikzstyle{Final}=[circle, accepting, thick, minimum size=0.6cm, inner 
sep=0cm,draw=black]
\tikzstyle{RState}=[circle, very thick, minimum size=0.8cm, inner 
sep=0cm,draw=red]
\tikzstyle{tran}=[draw,->,font=\small]
\tikzstyle{obstyle}=[rounded corners,fill=gray!20]

\title{Expectation Optimization with Probabilistic Guarantees in POMDPs\\ with 
Discounted-sum Objectives}

\author{Krishnendu Chatterjee$^1$, Adri\'an Elgy\"utt$^1$, Petr Novotn\'y$^1$, 
Owen 
Rouill\'e$^2$ \\
$^1$ Institute of Science and Technology Austria, Klosterneuburg, Austria \\
$\{$Krishnendu.Chatterjee, adrian.elgyuett, petr.novotny$\}$@ist.ac.at \\
$^2$ \'Ecole Normale Sup\'erieure de Rennes, Rennes, France\\
owen.rouille@ens-rennes.fr
}
\begin{document}

\maketitle

\begin{abstract}
	Partially-observable Markov decision processes (POMDPs) with discounted-sum 
	payoff
	are a standard framework to model a wide range of problems related to 
	decision making 
	under uncertainty. 
	Traditionally, the goal has been to obtain policies that optimize the 
	expectation of 
	the discounted-sum payoff. 
	A key drawback of the expectation measure is that even low probability 
	events with 
	extreme payoff can significantly affect the expectation, and thus the 
	obtained policies
	are not necessarily risk-averse.
	An alternate approach is to optimize the probability that the payoff is 
	above a certain 
	threshold, which allows obtaining risk-averse policies, but ignores 
	optimization of
	the expectation. 
	We consider the expectation optimization with probabilistic guarantee 
	(EOPG) problem,
	where the goal is to optimize the expectation ensuring that the payoff is 
	above a 
	given threshold with at least a specified probability.
	We present several results on the EOPG problem, 
	including the first algorithm to solve it.
\end{abstract}

\section{Introduction}
\noindent{\em POMDPs and Discounted-Sum Objectives.}
Decision making under uncertainty is a fundamental problem in artificial 
intelligence.
Markov decision processes (MDPs) are the {\em de facto} model that allows both 
decision-making choices as well as stochastic 
behavior~\cite{Howard,Puterman2005}.
The extension of MDPs with uncertainty about information gives rise to the 
model of  
partially-observable Markov decision processes 
(POMDPs)~\cite{LittmanThesis,PT87}.
POMDPs are used in a wide range of areas, such as planning~\cite{RN10}, 
reinforcement learning~\cite{LearningSurvey}, 
robotics~\cite{KGFP09,kaelbling1998planning},
to name a few.
In decision making under uncertainty, the objective is to optimize a payoff 
function.
A classical and basic payoff function is the {\em discounted-sum payoff}, where 
every 
transition of the POMDP is assigned a reward, and for an infinite path (that 
consists of
an infinite sequence of transitions) the payoff is the discounted-sum of the 
rewards 
of the transitions.

\smallskip\noindent{\em Expectation Optimization and Drawback.}
Traditionally, POMDPs with discounted-sum payoff have been studied, where the 
goal is to obtain policies that optimize the expected payoff. 
A key drawback of the expectation optimization is that it is not robust 
with respect to risk measures. 
For example, a policy $\sigma_1$ that achieves with probability 
$1/100$ payoff $10^4$ and with the remaining probability payoff~0 has higher 
expectation 
than a policy $\sigma_2$ that achieves with probability $99/100$ payoff~$100$ 
and 
with the remaining probability payoff~0. 
However the second policy is more robust and less risk-prone, and is desirable 
in 
many scenarios. 

\smallskip\noindent{\em Probability Optimization and Drawback.}
Due to the drawback of expectation optimization, there has been 
recent interest to study the optimization of the probability to 
ensure that the payoff is above a given threshold~\cite{HYV16:risk-pomdps}.
While this ensures risk-averse policies, it ignores the expectation optimization. 

\smallskip\noindent{\em Expectation Optimization with Probabilistic Guarantee.}
A formulation that retains the advantages of both the above optimization 
criteria,
yet removes the associated drawbacks, is as follows: given a payoff threshold 
$\thr$
and risk  bound $\rbound$, the objective is the expectation 
maximization
w.r.t. to all policies that ensure the payoff is at least $\thr$ with 
probability
at least $1-\rbound$.
We study this expectation optimization with probabilistic guarantee (EOPG) 
problem
for discounted-sum POMDPs. 

\smallskip\noindent{\em Motivating Examples.} We present some 
motivating examples for the EOPG formulation.
\begin{compactitem}
	\item {\em Bad events avoidance.} 
	Consider planning under uncertainty (e.g., self-driving cars) where certain 
	events are dangerous (e.g., the distance between two cars less than a 
	specified distance), and it must be ensured that such events happen with 
	low probability.
	Thus, desirable policies aim to maximize the expected payoff, ensuring the 
	avoidance of 
	bad events with a specified high probability.
	
	\item {\em Gambling.} In gambling, while the goal is to maximize the 
	expected profit,
	a desirable risk-averse policy would ensure that the loss is less than a 
	specified amount with high probability (say, with probability $0.9$).
\end{compactitem}
Thus, the EOPG problem for POMDPs with discounted-sum payoff is an important 
problem
which we consider.

\smallskip\noindent{\em Previous Results.} 
Several related problems have been considered, and two most relevant 
works
are the following:
\begin{compactenum}
	\item {\em Chance-constrained (CC) problem.}
	In the CC problem, 
	 certain bad states of the POMDP must not be reached with some 
	probability 
	threshold.
	That is, in the CC problem the probabilistic constraint is state-based 
	(some states
	should be avoided) rather than the execution-based discounted-sum 
	payoff.
	This problem was considered in~\cite{STW16:BWC-POMDP-state-safety}, but 
	only with deterministic policies. As already noted 
	in~\cite{STW16:BWC-POMDP-state-safety}, randomized (or mixed) policies are 
	more powerful.
	
	\item {\em Probability~1 bound.} The special case of the EOPG problem with 
	$\rbound=0$ has been considered in~\cite{CNPRZ17:BWC-POMDP}.
	This formulation represents the case
	with no risk. 
\end{compactenum}

\smallskip\noindent{\em Our Contributions.}
Our main contributions are as follows:
\begin{compactenum}
	\item {\em Algorithm.} 
	We present a randomized algorithm for approximating (up to any 
	given precision) the optimal solution to the EOPG problem.
	This is the first approach to solve the EOPG problem 
	for discounted-sum POMDPs.
	
	\item {\em Practical approach.} 
	We present a practical approach where certain 
	searches
	of our algorithm are only performed for a time bound. 
	This gives an {\em anytime} algorithm which approximates the probabilistic
	guarantee and then optimizes the expectation.
	
	\item {\em Experimental results.} 
	We present experimental results of our algorithm on classical POMDPs. 
\end{compactenum}
Due to space constraints, details such as full proofs are deferred to the 
appendix.

\smallskip\noindent \textbf{Related Works.}
POMDPs with discounted-sum payoff have been widely studied,
both for theoretical results~\cite{PT87,LittmanThesis} 
as well as practical tools~\cite{khl08,SV:POMCP,YSHL17:despot-jair}.
Traditionally, expectation optimization has been considered, and recent
works consider policies that optimize probabilities to 
ensure discounted-sum payoff above a threshold~\cite{HYV16:risk-pomdps}. 
Several problems related to the EOPG problem have been considered 
before: 
(a)~for probability threshold~1 for long-run average and stochastic shortest 
path 
problem in fully-observable MDPs~\cite{bfrr14,rrs15}; 
(b)~for risk bound~0 for discounted-sum payoff for POMDPs~\cite{CNPRZ17:BWC-POMDP}; and
(c)~for general probability threshold for long-run average payoff in 
fully-observable MDPs~\cite{ckk15}. 
The \emph{chance constrained}-optimization was studied 
in~\cite{STW16:BWC-POMDP-state-safety}. 
The general EOPG problem for POMDPs with discounted-sum payoff has not been 
studied before,
although development of similar 
objectives was proposed for perfectly observable MDPs~\cite{DefEW:08:risk-mdp}.
A related approach for POMDPs is called
\emph{constrained 
POMDPs}~\cite{UH10:constrained-pomdp-online,PMPKGB15:constrained-POMDP}, 
where the aim is to maximize the expected payoff ensuring that the expectation 
of some other quantity is bounded.
In contrast, in the EOPG problem the constraint is probabilistic rather than an 
expectation 
constraint, and as mentioned before, the probabilistic constraint ensures
risk-averseness
as compared to the expectation constraint. 
Thus, the constrained POMDPs and the EOPG problem, though related, consider 
different optimization criteria.

\section{Preliminaries}
\label{sec:prelims}
Throughout this work, we mostly follow standard (PO)MDP notations 
from~\cite{Puterman2005,LittmanThesis}.
%
We denote by $\distr(X)$ the set of all probability distributions on a finite 
set $X$, i.e. all functions $f: X \rightarrow [0,1]$ s.t. $\sum_{x\in 
	X}f(x)=1$. 

\begin{definition}\textbf{(POMDPs.)}
	A \emph{POMDP} is a
	tuple $\pomdp=(\states,\act,\trans,\reward,\obs,\obsmap,\initd)$ 
	where
	$\states$ is a finite set of \emph{states},
	$\act$ is a finite alphabet of \emph{actions},
	$\trans:\states\times\act \rightarrow \distr(\states)$ is a 
	\emph{probabilistic transition function} that given a state $s$ and an
	action $a \in \act$ gives the probability distribution over the successor 
	states, 
	$\reward: \states \times \act \rightarrow \reals$  is a reward 
	function,
	$\obs$ is a finite set of \emph{observations},
	$\obsmap:\states\rightarrow \distr(\obs)$ is a probabilistic 
	\emph{observation function} that 
	maps every state to a distribution over observations, and 
	$\initd\in \distr(\states)$ is the \emph{initial belief}.
	We abbreviate $\trans(s,a)(s')$ and $\obsfunc(s)(o)$ by 
	$\trans(s'|s,a)$ and $\obsfunc(o|s)$, respectively. 
\end{definition}

\noindent \textbf{Plays \& Histories.}
A \emph{play} (or an infinite path) in a POMDP is an infinite sequence $\rho = 
s_0 a_1 s_1 a_1 s_2 a_2 \ldots$ 
of states and actions s.t.
$s_0 \in \supp(\initd)$ and
for all $i \geq 0$ 
we have $\trans(s_{i+1}\mid s_{i-1},a_i)>0$. A \emph{finite path} (or just 
\emph{path}) is a finite prefix of a 
play ending with a state, i.e. a sequence from $(\states\cdot\act)^*\cdot 
\states$. 
%
A~\emph{history} is 
a finite sequence of actions and observations 
$\hist=a_1 \ob_1 \dots a_{i-1} \ob_i\in (\act\cdot\obs)^*$ 
s.t. there is a path $w=s_0 a_1 s_1 \dots a_{i} s_i$ with 
$\obsfunc(o_j\mid s_j)>0$ 
for each $1\leq j \leq i$.
We write $\hist=\histfunc(w)$ to indicate that history 
$\hist$ 
corresponds to a path $w$. The \emph{length} of a path (or history) $w$,
denoted by $\len{w}$, 
is the number of actions in $w$, and the length of a play $\rho$ is 
$\len{\rho}=\infty$.

\smallskip\noindent\textbf{Discounted Payoff.}
Given a play $\rho =s_0 a_1 s_1 a_2 s_2 a_3 \ldots$ and a discount 
factor $0 
\leq \discount < 1$, the \emph{finite-horizon discounted payoff} of $\rho$ for 
horizon $N$ is $
\discpath{\rho}{\discount,N} = \sum_{i=0}^{N} 
\discount^{i}\cdot\reward(s_i,a_{i+1}).$
The \emph{infinite-horizon discounted payoff} 
$\discsum_\discount$ of $\rho$ is
$\textstyle
\discpath{\rho}{\discount} = \sum_{i=0}^{\infty} 
\discount^{i}\cdot\reward(s_i,a_{i+1}).
$

\smallskip\noindent \textbf{Policies.}
A \emph{policy} (or \emph{strategy}) is a blueprint for selecting actions based on 
the past history. 
Formally, it 
is a function $\sigma$ 
which assigns to a history a probability distribution 
over the actions, i.e. $\sigma(\hist)(a)$ is the probability of selecting 
action $a$ after observing history $\hist$ (we abbreviate 
$\sigma(\hist)(a)$ to 
$\sigma(a\mid\hist)$). A policy is \emph{deterministic} if for each history 
$\hist$ the distribution $\sigma(\cdot\mid \hist)$ selects a single action with 
probability 1. For $\sigma$ deterministic we write $\sigma(\hist)=a$ to 
indicate that $\sigma(a\mid\hist)=1$.

\noindent\textbf{Beliefs.}
A \emph{belief} is a distribution on states (i.e. an element of 
$\distr(\states)$)
indicating the probability of being in each particular state given the current
history. The initial belief $\initd$ is given as a part of the POMDP.
Then, in each step, when the history observed so far is $h$, the current belief
is $\belseq{\belief}{h}$, an action $a\in \act$ is played, and an observation
$o\in \obs$ is received, the updated belief $\belseq{\belief}{h'}$ for history
$h'=hao$ can be computed by a standard Bayesian formula~\cite{cassandra1998exact}.

\noindent \textbf{Expected Value of a Policy.} Given a POMDP $\pomdp$, a policy 
$\sigma$, a horizon $N$, and a discount factor $\discount$, 
the \emph{expected value}
of $\sigma$ from $\initd$ is the expected 
value of the 
infinite-horizon discounted sum 
under policy $\sigma$ when starting in a state sampled from the initial belief 
of $\initd$ of $\pomdp$:
$
\evalue(\sigma) = \E_\initd^{\sigma}[\discpath{}{\discount,N}].
$

\noindent \textbf{Risk.} A \emph{risk level} 
$\risk(\sigma,\thr,\discpath{}{\discount,N})$ of a policy 
$\sigma$ at  
threshold 
$\thr\in\Rset$ w.r.t. payoff function $\discpath{}{\discount,N}$ is 
the probability that the payoff of a play 
generated by $\sigma$ is below $\thr$, i.e. 
\[
\risk(\sigma,\thr,\discpath{}{\discount,N}) = 
\probm^{\sigma}_{\initd}(\discpath{}{\discount,N}< \thr).
\]

\noindent \textbf{EOPG Problem.} We now define the problem of \emph{expectation 
optimization with probabilistic guarantees} (the \emph{EOPG} problem for short). We first define a finite-horizon variant, and then discuss the infinite-horizon version in Section 3. 
In EOPG problem, we are given a threshold $\thr\in\Rset$, a 
risk bound $\rbound$, and a horizon $N$. A policy $\sigma$ is a \emph{feasible 
solution} of the problem if $\risk(\sigma,\thr,\discpath{}{\discount,N})\leq 
\rbound$. The goal of the EOPG problem is to find a feasible solution $\sigma$ 
maximizing $\evalue(\sigma)$ among all feasible solutions, provided that 
feasible solutions exist.

\paragraph{Observable Rewards.} We solve the EOPG problem under the assumption that 
rewards in the POMDP are \emph{observable}. This means that 
$\reward(s,a)=\reward(s',a)$ whenever 
$\obsfunc(s)=\obsfunc(s')$ or if both $s$ and $s'$ have a positive probability 
under the initial belief. This is a natural assumption satisfied by many 
standard benchmarks~\cite{HYV16:risk-pomdps,CCGK15}. At the end of 
Section~\ref{sec:algo} we discuss how could be our results extended to 
unobservable rewards.

\noindent\textbf{Efficient Algorithms.} A standard way of making POMDP planning 
more efficient is to 
design an algorithm that is \emph{online} (i.e., it computes a local 
approximation of the $\eps$-optimal policy, selecting the 
best action for the current belief~\cite{RPPC08:online-planning-pomdp}) and 
\emph{anytime,} i.e. computing better 
and 
better approximation of the $\eps$-optimal policy over its runtime, returning a 
solution together with some guarantee on its quality if forced to terminate 
early.

\section{Relationship to CC-POMDPs}
We present our first result showing that an approximate infinite-horizon (IH) 
EOPG problem can be reduced to a finite-horizon variant.  While similar 
reductions are natural when dealing with discounted payoff, for the 
EOPG problem the reduction is somewhat subtle due to the presence of the risk 
constraint. We then show that the  EOPG problem can be reduced to 
chance-constrained POMDPs 
and solved using the RAO* algorithm~\cite{STW16:BWC-POMDP-state-safety}, but 
we also present several drawbacks of this approach.

Formally, we define the IH-EOPG problem as follows: we are given $\thr$ and 
$\rbound$ as before and in addition, an error term $\eps$. We say that an 
algorithm $\eps$-solves the IH-EOPG problem if, whenever the problem has a 
feasible solution (feasibility is defined as before, with 
$\discpath{}{\discount,N}$ replaced by $\discpath{}{\discount}$), the algorithm 
finds a policy $\sigma$ s.t. 
$\risk(\sigma,\thr-\eps,\discpath{}{\discount})\leq \rbound$ and 
$\E_\initd^{\sigma}[\discpath{}{\discount}]\geq \riskvalue{\thr}{\rbound} - 
\eps$, where 
$$\riskvalue{\thr}{\rbound}=\sup\{\E_\initd^{\pi}[\discpath{}{\discount}]\mid 
\risk(\pi,\thr,\discpath{}{\discount})\leq 
\rbound\}.$$

\paragraph{Infinite to Finite Horizon.} 
Let $\pomdp$ be a $\discount$-discounted POMDP, $\thr$ a payoff threshold, 
$\rbound\in[0,1]$ a risk bound, and $\eps>0$ an error term.
Let 
$N{(\eps)}$ be a horizon such that the following holds: 
$\discount^{N{(\eps)}}\cdot|\max\{0,\reward_{\max}\}-\min\{0,\reward_{\min}\}|\leq
(1-\discount)\cdot\frac{\eps}{2}$, where 
$\reward_{\max}$ and $\reward_{\min}$ are the maximal and minimal rewards 
appearing 
in~$\pomdp$, respectively.

\begin{lemma}
	\label{lem:finhon}
If there exists a feasible solution of the IH-EOPG problem with threshold $\thr$ and risk bound $\rbound$. Then there exists a policy $\sigma$ satisfying $\risk(\sigma,\thr-\frac{\eps}{2},\discpath{}{\discount,N(\eps)})\leq \rbound$.
Moreover, let $\sigma$ be an optimal solution to the EOPG problem with the risk bound $\rbound$, horizon $N(\eps)$, and with threshold $\thr' = \thr-\frac{\eps}{2}$.
Then 
$\sigma$ is an $\eps$-optimal solution to the IH-EOPG problem.
\end{lemma}

The previous lemma effectively shows that to solve an approximate version of 
the EOPG problem, it suffices to solve its finite 
horizon version.

\paragraph{Chance-Constrained POMDPs.} In 
the chance constrained (CC) optimization problem~\cite{STW16:BWC-POMDP-state-safety}, we are given a POMDP $\pomdp$, a finite-horizon bound $N\in 
\Nset$, and a set of constraint-violating states $X$, which is a subset of the 
set 
of states of $\pomdp$. We are also given a \emph{risk bound} $\Delta$. The goal is to optimize the expected 
finite-horizon 
payoff, i.e. the expectation of the following random variable:
\[\textstyle
\accpayoff{N}(s_0a_1s_1a_2s_2\ldots) = \sum_{i=0}^{N} \reward(s_i,a_{i+1}).
\]
The optimization is subject to a constraint that the probability of entering a 
state from $C$ (so-called \emph{execution risk}) stays below the risk bound 
$\Delta$. 

\paragraph{From EOPGs to CC-POMDPs.} 
We sketch how the FH-EOPG relates to CC-POMDP optimization. In the EOPG problem, the 
constraint violation occurs when the finite-horizon discounted payoff in step 
$N$ is smaller than a threshold $\thr$. To formulate this in a 
CC-POMDP setting, we need to make the constraint violation a property of a 
\emph{state} of the POMDP. 
Hence, we construct a new POMDP $\pomdp'$ with an 
extended state space: the states of $\pomdp'$ are triples of the form 
$\tilde{t}=(s,i,x)$, where $s$ is a state of the original POMDP $\pomdp$, 
$0\leq i \leq N(\eps)$ is a time index, and $x \in \Rset$ is a number 
representing the discounted reward accumulated before reaching the state 
$\tilde{t}$. The remaining  components of $\pomdp'$ are then extended in a 
natural way from $\pomdp$. By solving the CC-POMDP problem for $\pomdp'$, where 
the set $X$ contains extended states $(s,N,x)$ where $x<\thr$, 
we obtain a policy in $\pomdp'$ which can be carried back to $\pomdp$ where it 
forms an optimal solution of the FH-EOPG problem.

\paragraph{Discussion of the CC-POMDP Approach.} It follows that we could, in 
principle, 
reduce the EOPG problem to CC-POMDP optimization and then solve 
the latter using the known RAO* algorithm~\cite{STW16:BWC-POMDP-state-safety}. However, there are several issues 
with this approach.

First, RAO* aims to find an optimal 
\emph{deterministic} policy in CC-POMDPs. 
But as already mentioned in~\cite{STW16:BWC-POMDP-state-safety}, the 
optimal solution to the CC-POMDP (and thus also to EOPG) problem might require 
randomization, and deterministic policies may have arbitrarily worse expected 
payoff than 
randomized ones (it is well-known that randomization might be necessary for 
optimality in constrained (PO)MDPs, 
see~\cite{FS95:CMDP,KimLKP:11:point-based-constrained-POMDP,SprKT:14:path-constrained-MDPs}).
 
Second, although RAO* converges to an optimal constrained deterministic policy, 
it \emph{does not} provide \emph{anytime} 
guarantees about the risk of the policy it constructs. RAO* is an AO*-like 
algorithm that iteratively searches the belief space and in 
each step computes a \emph{greedy policy} that is optimal on the already 
explored fragment of the belief space. 
During its 
execution, RAO* works with an \emph{under-approximation} of a risk taken by the 
greedy policy: this is because an optimal risk to be taken in belief 
states that were not yet explored is under-approximated by an 
\emph{admissible heuristic}. So if the algorithm is stopped 
prematurely, the actual risk taken by the current greedy policy can be much 
larger than indicated by the algorithm. This is illustrated in the following 
example.

\begin{example}\label{examp1}
Consider the MDP in Figure~\ref{fig:noupbound}, and consider $\thr = 1$ and 
$\rbound=\frac{2}{3}$; that is, in the chance-constrained reformulation, we 
seek an optimal policy for which the probability of hitting $x$ and $y$ is at 
most 
$\frac{2}{3}$. Consider the execution of RAO* which explores states $s$, $t$, $v$, $w$, and $x$ (since the MDP is 
perfectly observable, we work directly with states) and then is prematurely terminated. 
(The order in which unexplored nodes are visited is determined by a heuristic, 
and in general we cannot guarantee that $u$ is explored earlier in RAO*'s 
execution). At this moment, the risk taken by an optimal 
policy from $u$ is under-approximated using an admissible heuristic. In case of 
using a \emph{myopic} heuristic, as suggested in~\cite{STW16:BWC-POMDP-state-safety}, the risk from $u$ 
is under-approximated by $0$. Hence, at this moment the best greedy policy satisfying the risk constraint is the one which plays action $b$ in state $t$: the risk-estimate of this policy is $\frac{1}{4}<\frac{2}{3}$ (the probability of reaching $x$). However, any deterministic policy that selects $b$ in state $t$ takes 
an overall risk $\frac{3}{4}>\frac{2}{3}$. 

\end{example}
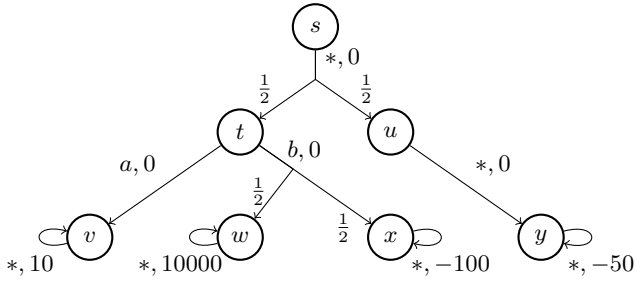
\begin{figure}
\centering
\begin{tikzpicture}[y=1.4cm]
\node[State] (s) at (0,0) {$s$};
\node[State] (t) at (-1,-1) {$t$};
\node[State] (u) at (1,-1) {$u$};
\node[State] (v) at (-3,-2) {$v$};
\node[State] (w) at (-1,-2) {$w$};
\node[State] (x) at (1,-2) {$x$};
\node[State] (y) at (3,-2) {$y$};

\path[tran] (s) -- node[pos=0.3,right] {$*,0$} +(0,-0.5) -- 
node[pos=0.7,label={[yshift=-0.2cm,xshift=0.2cm] 
135:{$\frac{1}{2}$}}] 
{} 
(t);
\path[tran] (s) --  +(0,-0.5) -- node[pos=0.7,label={[yshift=-0.2cm,xshift=-0.2cm] 
	45:{$\frac{1}{2}$}}] 
{} (u);
\path[tran] (u) -- node[auto] {$*,0$} (y);
\path[tran] (t) -- node[auto,swap] {$a,0$} (v);
\path[tran] (t) -- node[pos=0.25,label={[yshift=-0.2cm,xshift=-0.25cm] 
	45:{$b,0$}}] {} node[pos=0.3,coordinate,name=help] {} 
	node[pos=0.75,below]{$\frac{1}{2}$} (x);
\path[tran] (t) -- (help) --  node[left,pos=0.45] {$\frac{1}{2}$} (w);

\draw[tran, loop right] (x) to node[label={[yshift=0cm] 
	-90:{$*,-100$}}] {} (x);
\draw[tran, loop left] (v) to node[label={[yshift=0cm] 
	-90:{$*,10$}}] {} (v);

\draw[tran, loop left] (w) to node[label={[yshift=0cm] 
	-90:{$*,10000$}}] {} (x=w);
\draw[tran, loop right] (y) to node[label={[yshift=0cm] 
	-90:{$*,-50$}}] {} (y);
\end{tikzpicture}
\caption{A perfectly observable MDP with actions $a$, $b$ and with a discount 
factor 
$\discount=\frac{1}{2}$. The star denotes arbitrary action, and immediate 
reward of an action in a state is given next to the action label.}
\label{fig:noupbound}

\end{figure}

\newcommand{\nodepay}{V}
\newcommand{\noderisk}{r}
\newcommand{\noderiskbound}{B}
\newcommand{\nodeact}{a}

\section{Risk-Aware POMCP}
\label{sec:algo}
The previous example illustrates the main 
challenge in designing an online and anytime algorithm for the (finite 
horizon) EOPG problem: we need to keep upper bounds on the minimal risk 
achievable in the POMDP. Initially, the upper bound is $1$, and to decrease it, 
we need to discover a sufficiently large probability mass of paths that yield 
payoff above $\thr$. 

We propose an algorithm for the EOPG problem based on the popular 
POMCP~\cite{SV:POMCP}
planning algorithm: the
\emph{risk-aware POMCP} (or RAMCP for 
short). RAMCP solves the aforementioned challenge by performing, in each 
decision step, a large number of 
simulations using the POMCP heuristic to explore promising histories 
first. The 
key 
feature of RAMCP is that it extends POMCP with a new data 
structure, so called \emph{explicit tree}, which contains those histories 
explored during simulations that have payoff above the required threshold 
$\thr$. The explicit tree allows us to keep track of the upper bound on the 
risk that needs to be taken from the initial belief. After 
the simulation phase concludes, RAMCP uses the explicit tree to construct a 
perfectly observable tree-shaped \emph{constrained MDP}~\cite{Altman:book} encoding the 
EOPG problem on the explored fragment of the history tree of the input POMDP. 
The 
optimal 
distribution on actions is then 
computed using a linear program for constrained MDP 
optimization~\cite{Altman:book}. In 
the rest of this section, we present details of the algorithm and formally 
state its properties. 
In the following, we fix a POMDP $\pomdp$, a horizon $N$, a threshold 
$\thr$ 
and a risk bound $\rbound$.

\newcommand{\search}{\mathit{srch}}
\newcommand{\expl}{\mathit{exp}}
\newcommand{\prob}{p}

\begin{algorithm}[t]
	\DontPrintSemicolon
	\SetKwProg{proc}{procedure}{}{}
	\SetKwFunction{simulate}{Simulate}
	\SetKwFunction{update}{UpdateTrees}
	\SetKwFunction{rollout}{Rollout}
	\SetKwFunction{explore}{Explore}
	\SetKwFunction{select}{SelectAction}
	\SetKwFunction{play}{PlayAction}
	\SetKw{kwglobal}{global}
	
	
	\kwglobal 
	$\varthr,\varrboundr,n,\mathcal{T}_{\search},\mathcal{T}_{\expl}$\\
	\nl $\varthr\leftarrow \thr$;
	 $\varrboundr \leftarrow \rbound$;
	 $n\leftarrow N$;\\
	\nl $\mathcal{T}_{\search}\leftarrow\mathcal{T}_{\expl}\leftarrow 
	\mathit{empty~history }~\epsilon$ \\
	\nl initialize $\epsilon.U=1$\\
	\nl\While{$n>0$}{
		\nl \explore{$n$}\\
		\nl \select{}\\
		\nl $n\leftarrow n-1$
	}

\proc{\explore{$n$}}{ 
	\nl$h\leftarrow$ the root of $\mathcal{T}_{\search}$;
	$s\leftarrow$ sample from $b_h$\;
	\nl\lWhile{not timeout}{
			\simulate{$s,h,n,0$}
	}
}
	\caption{RAMCP.}
	\label{algo:ramcp}
\vspace{-0.25em}
\end{algorithm}

\paragraph{RAMCP.}
The main loop of RAMCP is pictured in 
Algorithm~\ref{algo:ramcp}. 
In each decision step, RAMCP performs a \emph{search} phase followed by 
\emph{action selection} followed by \emph{playing} the selected action (the 
latter two performed within $\mathtt{SelectAction}$ procedure). 
We describe the 
three 
phases separately.

\paragraph{RAMCP: Search Phase.} The search phase is shown in 
Algorithms~\ref{algo:ramcp} and \ref{algo:search2}. In the following, we first introduce the data structures the algorithm works with, then the elements of these structures, and finally we sketch how the search phase executes. \newcommand{\depth}{\mathit{depth}}
\newcommand{\pf}{\mathit{pay}}
\begin{algorithm}[!t]
	\DontPrintSemicolon
	\SetKwProg{proc}{procedure}{}{}
	\SetKwFunction{simulate}{Simulate}
	\SetKwFunction{update}{UpdateTrees}
	\SetKwFunction{rollout}{Rollout}
	\SetKwFunction{explore}{Explore}
	
	\setcounter{AlgoLine}{0}

	\proc{\simulate{$s,h,\depth,\pf$}}{
		\nl\If{$\depth = 0$}{
			\nl	\lIf{$\pf\geq \varthr$}{\update{$h$}}
			
			\nl	\Return{$0$}
		}
		\nl\If{$h\not \in \mathcal{T}_{\search}$}{
			\nl	add $h$ to $\mathcal{T}_{\search}$\\
			\nl	\Return{\rollout{$s,h,\depth,\pf$}}
		}
		\nl\label{alg:line1}$a \leftarrow \arg\max_{a} h.V_a + \econst\cdot 
		\sqrt{\frac{\log h.N}{h.N_a}}$\\
		\nl sample $(s',o,r)$ from $(\delta(\cdot\mid s,a),\obsfunc(\cdot\mid 
		s'),\reward(s,a))$\\
		\nl$R \leftarrow r + \discount\cdot$\simulate{$s',hao,\depth-1,\pf + 
		r$}\\
		\vspace{1mm}
		\nl$(h.N,h.N_a) \leftarrow (h.N+1,h.N_a + 1)$\\
		\nl$h.V_a \leftarrow h.V_a + \frac{R-h.V_a}{h.N_a}$\\
		\nl\Return{R}
	}
	\vspace{1mm}
	\setcounter{AlgoLine}{0}
	\proc{\rollout{$s,h,\depth,\pf$}}{
		\nl\If{$\depth = 0$}{
			\nl	\lIf{$\pf\geq \varthr$}{\update{$h$}}
			
			\nl\Return{$0$}
		}
		\nl choose $a\in \act$ uniformly at random\\
		\nl sample $(s',o,r)$ from $(\delta(\cdot\mid s,a),\obsfunc(\cdot\mid 
		s'),\reward(s,a))$\\
		\nl \Return{$r+\discount\cdot$\rollout{$s',hao,\depth-1,\pf + r$}}
	}
	\vspace{1mm}
	
	\setcounter{AlgoLine}{0}
	\proc{\update{$h$}}{
		\nl$h_{\mathit{old}}\leftarrow$ shortest prefix of $h$ not in 
		$\mathcal{T}_{\expl}$\\

		\nl\For{$g$ prefix of $h$ and strict extension of $h_{\mathit{old}}$}{
			\nl	add $g$ to $\mathcal{T}_{\expl}$\\
			\nl	let $g=\gpred ao$; add edge $(\gpred,g)$ to 
			$\mathcal{T}_{\expl}$\\
			\nl compute $\prob(\gpred,g)$ and $\rewards(\gpred,g)$ \label{algoline:bu}
		}
		\nl		$g\leftarrow h$; $R\leftarrow 0$\\
\nl		${h}.U\leftarrow 0$\\
		\nl\label{alg:line2}\Repeat{ $g=\text{empty history}$}{
			\nl	let $g=\gpred ao$ \label{algoline:dp1}\\
			\nl\ForEach{$a\in \act$}{
			\nl $\gpred.U_a \leftarrow 1-\displaystyle\sum_{o \in \obs} 
			\prob(\gpred,\gpred ao)\cdot 
			(1-{\gpred}ao.U)$}
			\nl $\gpred.U \leftarrow \min_{a\in \act}\gpred.U_a$ \label{algoline:dp2}\\
			\nl		\If{$g\not \in \mathcal{T}_{\search}$}{
				\nl		add $g$ to $\mathcal{T}_{\search}$\\
				\nl		$g.N\leftarrow 1$; $g.N_a\leftarrow 1$; $R \leftarrow R 
				+ \rewards(\gpred,g)$\\
				\nl		$g.V_a \leftarrow R$
			}
			\nl\label{alg:line3}	$g\leftarrow \gpred$
		
	}
}
	\caption{RAMCP simulations.}
	\label{algo:search2}
\vspace{-0.25em}
\end{algorithm}
\newcommand{\mdp}{\mathcal{M}}

\emph{Data Structures.} In the search phase, RAMCP 
explores, by performing simulations, the \emph{history tree} 
$\mathcal{T}_{\pomdp}$ of the input POMDP 
$\pomdp$. Nodes of the tree are the histories of $\pomdp$ of length 
$\leq N$. The tree is rooted in 
the 
empty history, and for each history $h$ of length at most $N-1$, each action 
$a$, and observation $o$, the node $h$ has a child $hao$. RAMCP works with two 
data structures, that are both sub-trees of $\mathcal{T}_{\pomdp}$: a search 
tree $\mathcal{T}_{\search}$ and explicit tree $\mathcal{T}_{\expl}.$
Intuitively, $\mathcal{T}_{\search}$ corresponds to the standard POMCP search 
tree while $\mathcal{T}_{\expl}$ is a sub-tree of $\mathcal{T}_{\search}$ 
containing histories leading to payoff above $\thr$. The term ``explicit'' 
stems from the fact that we explicitly compute beliefs and transition 
probabilities for the nodes in $\mathcal{T}_{\expl}$. Initially (before the 
first search phase), both $\mathcal{T}_{\search}$ and $\mathcal{T}_{\expl}$ 
contain a single node: empty history.

\emph{Elements of Data Structures.} Each node $h$ of 
$\mathcal{T}_{\search}$ has 
these attributes: for each action $a$ there is $h.V_a$, the average expected 
payoff obtained from the node $h$ after playing $a$ during past simulations, and $h.N_a$, the number of times action $a$ 
was selected in node $h$ in past simulations. Next, we have
$h.N$, the number of times the node was visited during past simulations. Each 
node $h$ of 
$\mathcal{T}_{\search}$ also contains a particle-filter approximation of the 
corresponding belief $b_h$. A node $h$ of the explicit tree has an attribute 
$h.U$, the upper bound on the risk from belief $b_h$, and, for each action $a$, attribute $h.U_a$, the upper bound on the risk when playing $a$ from belief $b_h$. Also, 
each node of $\mathcal{T}_{\expl}$ contains an exact representation of the 
corresponding belief, and each edge $(h,hao)$ of the explicit tree is labelled 
by numbers $\prob(h,hao) = \sum_{s,s'\in \states}b_h(s) \cdot \trans(s'\mid 
s,a)\cdot \obsfunc(o\mid s')$, i.e. by the probability of observing $o$ when 
playing action $a$ after history $h$, and $\rewards(h,hao)=\rewards(h,a)$, 
where $\rewards(h,a)$ is equal to $\reward(s,a)$ for any state $s$ with 
$b_{h}(s)>0$ (here we use the facts that rewards are observable).

 \emph{Execution of Search Phase.} Procedures $\mathtt{Simulate}$ and $\mathtt{Rollout}$ are basically 
 the same as in POMCP---within the search tree we choose actions heuristically 
 (in line~\ref{alg:line1} of $\mathtt{Simulate}$, the number $\econst$ is POMCP's 
 exploration constant), outside of it we choose actions uniformly at random. 
 However, whenever a simulation succeeds in surpassing the threshold $\thr$, we 
 add the observed history and all its prefixes to both the explicit and search 
 trees (procedure $\mathtt{UpdateTrees}$). Note that 
 computing $\prob(\gpred,g)$ on line~\ref{algoline:bu} entails computing full 
 Bayesian updates on the path from $h_{\mathit{old}}$ to $h$ so as to compute 
 exact beliefs of the corresponding nodes. Risk bounds for the 
 nodes corresponding to prefixes of $h$ are updated accordingly using a 
 standard dynamic programming update 
 (lines~\ref{algoline:dp1}--\ref{algoline:dp2}), starting from the newly added 
 leaf whose risk is $0$ (as it corresponds to a history after which $\thr$ is 
 surpassed).
 We have 
 the following:

\begin{lemma}
\label{lem:search-convergence}
\begin{compactenum}
\item At any point there exists a policy $\sigma$ such that 
$\probm^{\sigma}_{b_{h_{\mathit{root}}}}(\discpath{}{\discount,N-\len{h_{\mathit{root}}}}
< \varthr)\leq h_\mathit{root}.U$.
\item As $\mathit{timeout} \rightarrow \infty$, the probability that 
$h_\mathit{root}.U$ becomes equal to  
$\inf_{\sigma}\probm^{\sigma}_{b_{h_{\mathit{root}}}}(\discpath{}{\discount,N-\len{h_{\mathit{root}}}}
< \varthr)$ before $\timeout$ expires converges to $1$.
\end{compactenum}
\end{lemma}
\begin{sproof}
For part (1.) we prove the following stronger statement: Fix any point 
of algorithm's execution, and let $L$ be the length of $h_{\mathit{root}}$ (the 
history at the root of $\trsearch$ and $\trexpl$) at 
this point. Then for any node $\varh$ of $\trexpl$ there exists a policy 
$\sigma$ 
s.t. 
$\probm^{\sigma}_{b_\varh}(\discpath{}{\discount,\len{\varh}-N}
< (\varthr - 
(\discpath{\varh}{\discount,N}-\discpath{\hroot}{\discount,N})\cdot 
\gamma^{-\len{f}})/\gamma^{\len{\varh}-L})\leq 
\varh.U$. The proof
proceeds by a rather straightforward induction. The statement of the lemma then 
follows by plugging the root of 
$\mathcal{T}_{\expl}$ into $\varh$. 

For part (2.), the crucial observation is that as $\timeout\rightarrow \infty$, 
with probability converging to $1$ the tree $\trexpl$ will at some point 
contain all histories of length $N$ (that have $\hroot$ as a prefix) whose 
payoff is above the required threshold. It can be easily shown that at such a 
point $\hroot.U = 
\inf_{\sigma}\probm^{\sigma}_{b_{h_{\mathit{root}}}}(\discpath{}{\discount,N-\len{h_{\mathit{root}}}}
< \varthr)$, and $\hroot.U$ will not change any further.
%
\end{sproof}


\paragraph{RAMCP: Action Selection.}
\newcommand{\constr}{C}
The action selection phase is sketched in Algorithm~\ref{algo:select}. If the 
current risk bound is $1$, there is no real constraint and we select 
an action maximizing the expected payoff. Otherwise, to compute a distribution 
on actions to select, we construct and solve a certain constrained MDP.

\emph{Constructing Constrained MDP.}
RAMCP first computes a \emph{closure} $\treeclos$ of $\mathcal{T}_{\expl}$. That is, first we set $\treeclos \leftarrow \trexpl$ and then for 
each node $h\in \mathcal{T}_{\expl}$ and each action $a$ such that $h$ has a 
successor of the form $hao\in\trexpl$ (in such a case, we say that $a$ is 
\emph{allowed} 
in $h$), the algorithm checks if there exists a successor of the form 
$hao'$ that is not in $\mathcal{T}_{\expl}$; all such ``missing'' successors of 
$h$ under $a$ are added to $\treeclos$. Such a tree  
$\treeclos$ defines a perfectly observable \emph{constrained MDP} 
$\mdp$: 
\begin{compactitem}
	\item	
the states of $\mdp$ are the nodes of $\treeclos$; 
\item 
for each internal node $h$ of $\treeclos$ and each action $a$ allowed in $h$ there 
is probability $\prob(h,hao)$ of transitioning from $h$ to $hao$ under 
$a$ (these probabilities sum up to $1$ for each $h$ and $a$ thanks to computing 
the closure). If $h$ is a leaf of $\treeclos$ and $\len{h}<N$, playing any action $a$ in $h$ leads with probability $1$ to a new 
 sink state $\mathit{sink}$ ($\mathit{sink}$ has self-loops under all actions). 
\item 
Rewards in $\mdp$ are given by the function 
$\rewards$; self-loop on the sink and state-action pairs of the form 
$(h,a)$ with $\len{h}=N$ have reward $0$. Transitions from the other leaf nodes 
$h$ to the sink state have reward $\max_{a}h.V_a$. That is, from nodes that 
were never explored 
explicitly (and thus have $U$-attribute equal to $1$) we estimate the optimal 
payoff by previous POMCP simulations.
\item
We also have a constraint function $\constr$ assigning penalties to state-action pairs: $\constr$ 
assigns $1/\discount^{N-\len{\hroot}}$ to pairs $(h,a)$ such that $h$ is a leaf 
of $\mathcal{T}_{\expl}$ of length $N$, and 
 $0$ to all other state-action pairs. 
\end{compactitem}
 
 \emph{Solving MDP $\mdp$.}
 Using a linear programming 
formulation of constrained MDPs~\cite{Altman:book}, RAMCP computes a randomized policy 
$\pi$ in $\mdp$ maximizing $\E^{\pi}[\discpath{}{\discount}]$ under the 
constraint $\E^{\pi}[\discpath{}{\discount}^\constr] \geq 1-\varrboundr$, where $\discpath{}{\discount}^\constr$ is a discounted sum of incurred penalties. The distribution $\pi$ is then the distribution on actions used by $\pi$ in the first step. An examination of the LP in~\cite{Altman:book} shows that each solution of the LP yields not only the policy $\pi$, but also for each action $a$ allowed in the root, a 
\emph{risk vector} $d^a$, i.e. an $|\obs|$-dimensional vector such that $d^a(o)=\probm^{\pi}_{\mdp}(\discpath{}{\discount}^\constr> 0 \mid \text{$a$ is played and $o$ is received})$.

\begin{remark}[Conservative risk minimization.] 
Note that $\mdp$ might have no policy satisfying the penalty
constraint. This happens when the $U$-attribute of the root
is greater than $\varrboundr$. 
In such a case, the algorithm falls back to a policy that minimizes the risk, 
which means choosing action $a$ minimizing $\mathit{root}.U_a$ 
(line~\ref{algo:riskmin} of \texttt{SelectAction}).
In such a case, all $d^a(o)$ are set to zero, to enforce that in the following 
phases 
the algorithm behaves conservatively (i.e., keep minimizing the risk).
\end{remark}

\begin{remark}[No feasible solution.]
When our algorithm fails to obtain a feasible solution, it  ``silently'' falls 
back to a risk-minimizing policy. This might not be the preferred option for 
safety-critical applications. 
However, the algorithm exactly recognizes when it cannot guarantee meeting 
the original risk-constraint---this happens exactly when at the entry to the 
\texttt{SelectAction} procedure, the $U$-attribute in the root of $\trexpl$ is 
$>\varrboundr$.
Thus, our algorithm has two desirable properties: 
(a) it can report that it has not obtained a feasible solution; 
(b) along with that, it presents a risk-minimizing policy. 
\end{remark}
\begin{algorithm}[t]
	\DontPrintSemicolon
	\SetKwProg{proc}{procedure}{}{}
	\SetKwFunction{simulate}{Simulate}
	\SetKwFunction{update}{UpdateTrees}
	\SetKwFunction{rollout}{Rollout}
	\SetKwFunction{explore}{Explore}
	\SetKwFunction{select}{SelectAction}
	\SetKwFunction{play}{PlayAction}
	
	\proc{\select{}}{
		\nl \If{$\varrboundr<1 \wedge  \mathit{root}.U < 1$}{
			\nl$\mdp\leftarrow$ constrained MDP determined by 
			$\mathcal{T}_{\expl}$\\
			\nl\If{$\mathit{root}.U \leq \varrboundr$}{
			\nl $d_\pi,\{d^a\}_{a\in\act}\leftarrow$ solve LP formulation 
			of $\mdp$}\nl\Else{\nl$d_\pi,\{d^a\}_{a\in\act}\leftarrow$ solve risk-minimizing variant \label{algo:riskmin}
			of $\mdp$  }
			\nl$a\leftarrow$ sample from $d_{\pi}$\\
		}
	
	\nl\Else{\nl $a\leftarrow \arg\max_a \mathit{root}.V_a$\\
		\nl $d^a(o)=1$ for each $o$\\

	}
	\nl\play{$a,d^a$}\\}
	\vspace{1mm}
	\setcounter{AlgoLine}{0}
	
		\proc{\play{$a,d^a$}}{
		\nl play action $a$ and receive observation $o$ and reward $R$\\ 
		\nl $\varthr\leftarrow (\varthr - R)/\discount $;
		$\;\varrboundr\leftarrow d^a(o)$\\
		\nl $h \leftarrow$ root of $\mathcal{T}_{\search}$ (and 
		$\mathcal{T}_{\expl}$)\\
		\nl $\mathcal{T}_{\search}\leftarrow $ subtree of 
		$\mathcal{T}_{\search}$ rooted in $hao$\\
		\nl $\mathcal{T}_{\expl}\leftarrow $ subtree of $\mathcal{T}_{\expl}$ 
		rooted in $hao$\\

	}
	\caption{RAMCP: action selection and play.}
	\label{algo:select}
\vspace{-0.25em}
\end{algorithm}
\begin{figure*}[!t]
	\begin{center}
		\includegraphics[width=0.33\textwidth]{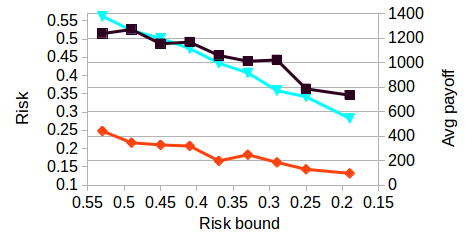}
		%
		\includegraphics[width=0.33\textwidth]{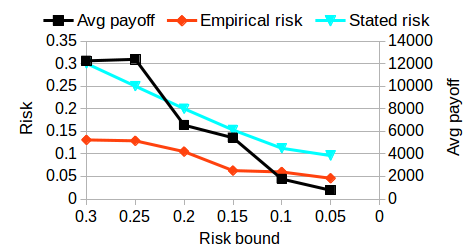}
		%
		\includegraphics[width=0.33\textwidth]{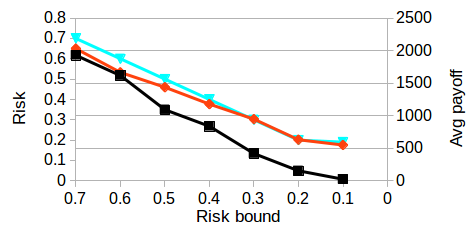}
	\end{center}
	\caption{Plots of results obtained from simulating (1.) the larger hallway POMDP benchmark (left), (2.) the MDP hallway benchmark (middle), and (3.) the smaller hallway POMDP benchmark (right). The horizontal axis represents a risk bound $\rbound$.
	}
	\label{fig:plots-benchmarks}
\vspace{-1mm}
\end{figure*}
\begin{lemma}
\label{lem:strat-convergence}
Assume that the original EOPG problem has a feasible solution. For a suitable exploration constant $\econst$, as $\mathit{timeout}\rightarrow 
\infty$, the distribution $d_\pi$ converges, with probability 1, to 
a distribution on actions used in the first step by some optimal solution to 
the EOPG problem.
\end{lemma}
\begin{sproof}
Assuming the existence of a feasible solution, we show that at 
the time point in which the condition in Lemma~\ref{lem:search-convergence} 
(2.) holds (such an event happens with probability converging to $1$), the 
constrained MDP $\mdp$ has a feasible solution. It 
then remains to prove that the optimal constrained payoff achievable in $\mdp$ 
converges to the optimal risk-constrained payoff achievable in $\pomdp$. Since 
rewards in $\mdp$ are in correspondence with rewards in $\trexpl$, it suffices 
to show that for each leaf $h$ of $\trexpl$ with $\len{h}<N$ and for each 
action $a$ the attribute $h.V_a$ converges with probability 1 to the optimal 
expected payoff for horizon $N-\len{h}$ achievable in $\pomdp$ after playing 
action $a$ from belief $b_h$. But since the $V_a$ attributes are updated  
by POMCP simulations, this follows (for a suitable exploration constant) from 
properties of POMCP (Theorem~1 in~\cite{SV:POMCP}, see also~\cite{KS06}).
\end{sproof}

\noindent{\bf RAMCP: Playing an Action.} The action-playing phase is shown in 
Algorithm~\ref{algo:select}. An action is played in the actual POMDP $\pomdp$ and 
a new observation $o$ and reward $R$ are obtained, and $\varthr$ and 
$\varrboundr$ are 
updated. Then, both the tree data structures are pruned so that the 
node corresponding to the previous history extended by $a,o$ becomes the new 
root of the tree. After this, we proceed to the next decision step.

\begin{theorem}
\label{thm:convergence}
Assume that an EOPG problem instance with a risk bound $\rbound$ and threshold 
$\thr$ has a feasible solution. 
As $\mathit{timeout}\rightarrow\infty$, the 
probability that RAMCP returns a payoff smaller than $\thr$ converges to a number 
$\leq\rbound$. For a suitable exploration constant $\econst$, the 
expected return of a RAMCP execution converges 
to $\riskvalue{\thr}{\rbound}$. 
\end{theorem}
\begin{sproof}
The proof proceeds by an induction on the length of the horizon $N$, using 
Lemma~\ref{lem:search-convergence} (2.) and Lemma~\ref{lem:strat-convergence}.
\end{sproof}

\noindent
RAMCP also provides the following anytime guarantee.

\begin{theorem}
\label{thm:safety}
Let $u$ be the value of the $U$-attribute of the root of 
$\mathcal{T}_{\expl}$ after the end of the first search phase of RAMCP execution. Then the probability that the remaining execution of 
RAMCP returns a payoff smaller than $\thr$ is at most $\max\{u,\rbound\}$. 
\end{theorem}

\noindent\textbf{Unobservable Rewards.} RAMCP could be adapted to work with 
unobservable rewards, at the cost of more computations. The difference is 
in 
the construction of the constraint function $C$: for unobservable rewards, 
the same history of observations and actions might encompass both paths 
that have payoff above the threshold and paths that do not. Hence, we 
would need to compute the probability of paths corresponding to a 
given branch that satisfy the threshold condition. This could be achieved by 
maintaining beliefs over accumulated payoffs.

\section{Experiments}
We implemented RAMCP on top of the POMCP implementation in 
AI-Toolbox~\cite{ait} and tested on three sets of benchmarks. 
The first two are the classical Tiger~\cite{kaelbling1998planning} and Hallway~\cite{SS04} benchmarks naturally 
modified to contain a risk taking aspect. In our variant of the Hallway 
benchmark, we again have a robot 
navigating a grid maze, oblivious to the exact heading and coordinates 
but able to sense presence of walls on
neighbouring cells. Some cells of the maze are \emph{tasks}. Whenever such a 
task cell is entered, the robot attempts to perform the task. When performing 
a task, there is a certain probability of a good outcome, after which a 
positive reward is gained, as well as a chance of a bad outcome, after which a 
negative penalty is incurred. There are different types of tasks in the maze 
with various expected rewards and risks of bad outcomes. Once a task is 
completed, it disappears from the maze. There are also 
``traps'' that probabilistically spin the robot around. 

As a third benchmark we consider an MDP variant of the Hallway benchmark.
Since the Tiger benchmark is small, we present results for the larger 
benchmarks. Our implementation and the benchmarks are available 
on-line.\footnote{https://git.ist.ac.at/petr.novotny/RAMCP-public}

We ran RAMCP benchmarks with different risk thresholds, starting with unconstrained POMCP 
and progressively decreasing risk until RAMCP no longer finds a feasible solution. 
For each risk bound we average outcomes of 1000 executions. 
In each execution, we used a timeout of 5 seconds in the first decision step and 
0.1 seconds for the remaining steps. Intuitively, in the first step the agent is 
allowed a ``pre-processing'' phase before it starts its operation, trying to 
explore as much as possible. 
Once the agent performes the first action, it aims to select actions as fast 
as possible.
We set the exploration constant to $\approx 2\cdot X$, where $X$ is the 
difference between 
largest and smallest undiscounted payoffs achievable in a given instance.
The test configuration was CPU: Intel-i5-3470, 3.20GHz, 4 cores; 
8GB RAM; OS: Linux Mint 18 64-bit. 

\paragraph{Discussion.} In Figure~\ref{fig:plots-benchmarks}, we present results of three of the benchmarks: (1.) The Hallway POMDP benchmark ($|\states|=67584,|\act|=3,|\obs|=43,\discount=0.95,N=30,\thr=10$); (2.) the perfectly observable version of our Hallway benchmark ($|\states|=4512,|\act|=3,|\obs|=1,\discount=0.98,N=120,\thr=10$); and (3.) smaller POMDP instance of the Hallway benchmark ($|\states|=7680$, $|\act|=3$,$|\obs|=33,\discount=0.8,N=35,\thr=10$).
In each figure, the $x$ axis represents the risk bound $\rbound$ -- the 
leftmost number is typically close to the risk achieved in POMCP trials. For 
each $\rbound$ considered we plot the following quantities: average payoff 
(secondary, i.e. right, $y$ axis), empirical risk (the fraction of trials in 
which RAMCP returned payoff smaller than $\thr$, primary $y$ axis) and stated 
risk (the average of $\max\{\rbound,U$-value of the root of $\tree_{\expl}$ 
after first search phase$\}$, primary axis). As expected, the stated risk 
approximates a lower bound on the empirical risk. Also, when a risk bound is 
decreased, average payoff tends to decrease as well, since the risk bound 
constraints the agent's behaviour. This trend is somewhat violated in some 
datapoints: this is because in particular for larger benchmarks, the timeout 
does not allow for enough exploration so as to converge to a tight 
approximation of the optimal policy. The main obstacle here is the usage of 
exact belief updates within the explicit tree, which is computationally 
expensive. An interesting direction for the future is to replace these updates 
with a 
particle-filter approximation (in line with POMCP) and thus increase search 
speed in exchange for weaker theoretical guarantees. Nonetheless, already the 
current version of RAMCP demonstrates the ability to perform risk vs. 
expectation trade-off in POMDP planning. 

\paragraph{Comparison with deterministic policies.} 
As illustrated in Example~\ref{examp1}, the difference of values for randomized 
vs deterministic policies can be large.
We ran experiments on Hallway POMDP benchmarks to compute deterministic 
policies (by computing, in action selection phase, an optimal deterministic 
policy in the constrained MDP $\mdp$, which entails solving a MILP problem). 
For instance, in benchmark (1.) for $\rbound=0.41$ the deterministic policy yields expected payoff $645.017$ compared to  $1166.8$ achieved by randomized policy. 
In benchmark (3.) with $\rbound=0.3$ we have expected payoff $107.49$ for deterministic vs. $695.81$ for randomized policies.

\section{Conclusion}
In this work, we studied the  expected payoff optimization with probabilistic 
guarantees in POMDPs. We introduced an online algorithm with anytime risk 
guarantees for the EOPG problem, implemented this algorithm, and tested it on 
variants of classical benchmarks. Our experiments show that our algorithm, 
RAMCP, is able to perform risk-averse planning in POMDPs.

\section*{Acknowledgements}
The research presented in this paper was supported by the Vienna Science and 
Technology Fund (WWTF) grant ICT15-003; Austrian Science Fund (FWF): S11407-N23 
(RiSE/SHiNE); and an ERC Start Grant (279307: Graph Games).

\bibliographystyle{named}
\bibliography{bibliography-master,new}

\clearpage
\appendix
\begin{center}
	{\Large Technical Appendix}
\end{center}
\section{Proof of Lemma~\ref{lem:finhon}}
First, assume that there exists a feasible solution $\sigma$ of the 
infinite-horizon problem. Due to the choice of $N(\eps)$, each play 
$\rho$ that satisfies $\discpath{\rho}{\discount,N(\eps)} \leq 
\thr-\frac{\eps}{2}$ also 
satisfies  $\discpath{\rho}{\discount} \leq \thr$, so $\sigma$ satisfies 
$\risk(\sigma,\thr-\frac{\eps}{2},\discpath{}{\discount,N(\eps)})\leq \rbound$.

	Assume that $\sigma$ has the desired properties. First we show that 
	$\probm^{\sigma}_{\initd} (\discpath{}{\discount} \leq 
	\thr-{\eps}) \leq \rbound$. But due to the choice of $N(\eps)$, each play 
	$\rho$ that satisfies $\discpath{\rho}{\discount} \leq \thr-\eps$ also 
	satisfies  $\discpath{\rho}{\discount,N(\eps)} \leq \thr-\frac{\eps}{2}$, 
	from 
	which the desired inequality easily follows. 
	
	Next, assume, for the sake of contradiction, that there is a policy $\pi$ 
	such 
	that $\probm^{\pi}_{\initd} 
	(\discpath{}{\discount} \leq 
	\thr) \leq \rbound$ and $\E^{\pi}_{\initd}[\discpath{}{\discount}] > 
	\E^{\sigma}_{\initd}[\discpath{}{\discount}] + \frac{\eps}{2}$. Each 
	play $\rho$ that satisfies $\discpath{\rho}{\discount,N(\eps)} \leq 
	\thr-\frac{\eps}{2}$ also satisfies $\discpath{\rho}{\discount} \leq \thr$, 
	from which it follows that 
	$\probm^{\pi}_{\initd}(\discpath{\rho}{\discount,N(\eps)}) \leq 
	\probm^{\pi}_{\initd}(\discpath{\rho}{\discount} \leq \thr)\leq \rbound$. 
	Moreover, $|\E^{\pi}_{\initd}[\discpath{\rho}{\discount}] - 
	\E^{\pi}_{\initd}[\discpath{\rho}{\discount,N(\eps)}]|\leq \frac{\eps}{2}$, 
	from which it follows that 
	$\E^{\pi}_{\initd}[\discpath{\rho}{\discount,N(\eps)}] > 
	\E^{\sigma}_{\initd}[\discpath{\rho}{\discount,N(\eps)}]$, a contradiction 
	with 
	the constrained optimality of $\sigma$.

\section{Proof of Lemma~\ref{lem:search-convergence}}

\renewcommand{\varh}{f}

 Before we proceed with the proof, we fix additional notation. For a history 
 $h$ and $0\leq j \leq \len{h}$ we denote by $h_j$ the and $h_{-j}$ the prefix 
 of $h$ of length $j$. Next, denote $h_{-j}$ the history obtained by removing 
 prefix of length $j$ from $h$. We also denote $\discpath{h}{\discount} = 
 \sum_{j=0}^{\len{h}-1} \discount^j\cdot\rewards(h_j,h_{j+1})$. 
 
%
%

\textbf{Proof of part (1.)}. We prove a more general statement: Fix any point 
of algorithm's execution, and let $L$ be the length of $h_{\mathit{root}}$ at 
this point. Then for any node $\varh$ that is at this point a node of the 
explicit tree there exists a policy $\sigma$ whose risk threshold at $(\varthr 
- \discpath{\varh_{-L}}{\discount})/\gamma^{\len{\varh}-L}$ when starting from 
belief $b_\varh$ and playing for $N-\len{\varh}$ steps is at most $\varh.U$, 
formally  $\probm^{\sigma}_{b_\varh}(\discpath{}{\discount,\len{\varh}-N}
< (\varthr - \discpath{\varh_{-L}}{\discount})/\gamma^{\len{\varh}-L})\leq 
\varh.U$. The statement of the lemma then follows by plugging the root of 
$\mathcal{T}_{\expl}$ into $\varh$. 

Before the very first call of procedure $\texttt{Simulate}$, the statement 
holds, as $\mathcal{T}_{\expl}$ only contains an empty history $\epsilon$ with 
trivial upper risk bound $\epsilon.U$. 

Next, assume that the statement holds \emph{after} each execution of the  
$\texttt{Simulate}$ procedure. Then, in procedure $\texttt{PlayAction}$, 
$\mathcal{T}_{\expl}$ is pruned so that now it contains a sub-tree of the 
original tree. At this point $L$ is incremented by one but at the same time 
$\varthr$ is set to $(\varthr-\rewards(o,a))/\gamma$, where $oa$ is the common 
prefix for all histories that remain in $\mathcal{T}_{\expl}$ after the 
pruning. Hence, for all such histories $\varh$ the term $(\varthr - 
\discpath{\varh_{-L}}{\discount})/\gamma^{\len{\varh}-L}$ is unchanged in the 
$\texttt{PlayAction}$ procedure and the statement still holds. 
Hence, it is sufficient to prove that the validity of the statement is 
preserved whenever a new node is added to $\mathcal{T}_{\expl}$ or the 
$U$-attribute of some node in $\mathcal{T}_{\expl}$ is changed inside 
procedure~\texttt{UpdateTrees}.

Now when a new $h$ with $\len{h}=N$ is added to $\mathcal{T}_{\expl}$, it is 
node that corresponds to a history $h$ such that all paths consistent with 
$h_{-L}$ have reward at least $\varthr$. For such $h$, the terms $(\varthr - 
\discpath{h_{-L}}{\discount})/\gamma^{\len{h}-L}$ and $N-\len{h}$ evaluate to 
$0$ and thus the statement holds for $h.U=0$.

So finally, suppose that $\varh.U$ is changed on line~\ref{algoline:dp2}. Let 
$a$ be the action realizing the minimum. For each child $\varh' = \varh a o'$ 
of $\varh$ in the explicit tree there is a policy $\sigma_{\varh'}$ 
surpassing threshold $(\varthr - 
\discpath{\varh'_{-L}}{\discount})/\gamma^{\len{\varh'}-L}=\rewards(o,a)+\discount\cdot(\varthr
 - \discpath{\varh_{-L}}{\discount})/\gamma^{\len{\varh}-L}$ with probability 
$\geq 1-\varh'.U$. By selecting action $a$ in $f$ and then continuing with 
$\sigma_{\varh'}$ when observation $o'$ is received (if we receive observation 
$o'$ s.t. $\varh a o'$ is not in $\mathcal{T}_{\expl}$, we can continue with an 
arbitrary policy), we get a policy $\sigma$ with the desired property. This 
is because of the dynamic programming update on the previous line.

\textbf{Proof of part (2.)}. We again start by fixing some notation. Fix any 
point in execution of the $\texttt{Search}$ procedure. Let $\hroot$ be the 
current root of $\trsearch$. The
\emph{safe} sub-tree of $\mathcal{T}_\pomdp$ rooted in $\hroot$ is a sub-tree 
$\trsafe(\hroot)$ of
$\mathcal{T}_\pomdp$ satisfying the following property: a history $\varh$ 
belongs 
to $\trsafe$ if and only if $\hroot$ is a prefix of $\varh$ and at the same 
time 
$\varh$ can be extended into a history $h$ of length 
$N$ such that $\discpath{\varh_{-\len{\hroot}}}{\discount} \geq \thr$ (in 
particular, all such histories $h$ belong to $\trsafe(\hroot)$). That is 
$\trsafe$ contains exactly those histories that lead to surpassing the current 
threshold.

We start with the following lemma, which will be also handy later.

\begin{lemma}
\label{lem:inf-exploration}
Let $\timeout = \infty$.
Fix a concrete call of procedure $\texttt{Search}$, and let $\hroot$ be the 
root of $\trsearch$ and $\trexpl$ in this call. Then, with probability $1$, 
each node of the sub-tree of $\mathcal{T}_{\pomdp}$ rooted is visited in 
infinitely many calls of procedure $\texttt{Simulate}$.
\end{lemma}
\begin{proof}
Assume, for the sake of contradiction, that there exists a node $\varh$ of 
$\mathcal{T}_{\pomdp}$ that is visited only in finitely many calls with 
positive probability. Let $\varh$ be such a node of minimal length. It cannot 
be that $\varh=\hroot$, since $\hroot$ is visited in each call of 
$\texttt{Simulate}$ and when $\timeout=\infty$, there are infinitely many such 
calls. So $\varh=hao$ for some $a,o$. Due to our assumptions, $h$ is visited in 
infinitely many calls of $\texttt{Simulate}$. This means that $h$ is eventually 
added to $\trsearch$. Now assume that $a$ is selected on line~\ref{alg:line1} 
infinitely often with probability $1$. Since there is a positive probability of 
observing $o$ after 
selecting $a$ for history $h$, this would been that $\varh$ is also visited 
infinitely often with probability 1, a contradiction. But the fact that $a$ is 
selected infinitely often with probability $1$ stems from the fact that $a$ is 
sampled according to POMCP simulations. POMCP is essentially the UCT algorithm 
applied to the history tree of a POMDP, and UCT, when run indefinitely, 
explores each node of the tree infinitely often (Theorem~4 in~\cite{KS06}).

\end{proof}

We proceed with the following lemma.

\begin{lemma}
\label{lem:treeconv}
Fix a concrete call of procedure $\texttt{Search}$, and let $\hroot$ be the 
root of $\trsearch$ and $\trexpl$ in this call. Then, as $\timeout \rightarrow 
\infty$, the probability that $\trexpl$ becomes equal to 
$\trsafe(\hroot)$ before $\timeout$ expires converges $1$.
\end{lemma}
\begin{proof}
We prove a slightly different statement: if $\timeout=\infty$, then the 
probability that $\trexpl$ eventually becomes equal to 
$\trsafe(\hroot)$ is 1. Clearly, this entails the lemma, since 
\begin{align*}
&\probm(\trexpl \text{ becomes equal to }
\trsafe(\hroot)) \\ &= \sum_{i=0}^{\infty}\probm(\trexpl \text{ becomes equal 
to 
} 
\trsafe(\hroot) \\ &~~~~~~~~~~~~~~~~\text{ in $i$-th call of \texttt{Simulate} 
}).
\end{align*}
(Here, $\probm$ denotes the probability measure over executions of our 
randomized algorithm).

So let $\timeout = 
\infty$.
From Lemma~\ref{lem:inf-exploration} it follows that with probability $1$, each 
node of the history tree is visited infinitely often. In particular, each node
representing history $h$ of length  
$N$ such that $\discpath{\varh_{-\len{\hroot}}}{\discount} \geq \thr$ is 
visited, with probability 1, in at least one call of procedure
$\texttt{Simulate}$. Hence , with probability 1, this node and all its 
predecessors, are 
added to $\trexpl$ during the sub-call $\texttt{UpdateTrees(h)}$. 
\end{proof}

Lemma~\ref{lem:search-convergence} then follows from the previous and the 
following lemma.

\newcommand{\genrisk}{\Psi}

\begin{lemma}
Assume that during that during some call of $\texttt{Simulate}$ it happens that 
$\trexpl$ becomes equal to $\trsafe(\hroot)$. Then at this point it holds that 
$\hroot.U = 
\inf_{\sigma}\probm^{\sigma}_{b_{h_{\mathit{root}}}}(\discpath{}{\discount,N-\len{h_{\mathit{root}}}}
< \varthr)$, and $\hroot.U$ will not change any further. 
\end{lemma}
\begin{proof}
	Let $L=\len{\hroot}$.
	We again prove a somewhat stronger statement: given assumptions of the 
	lemma, for each $h\in \trexpl = \trsafe(\hroot)$ it holds that $h.U = 
	\inf_{\sigma}\probm^{\sigma}_{b_{h}}(\discpath{}{\discount,N-\len{h}}
	< \varthr -  \discpath{h_{-L}}{\discount})/\gamma^{\len{h}-N})$. Since an 
	easy induction shows that $U$-attributes can never increase, from part (1.) 
	of Lemma~\ref{lem:search-convergence} we get that 
	once this happens, the $U$ attributes of all nodes in $\trexpl$ now 
	represent the minimal achievable risk at given thresholds and thus these 
	attributes can never 
	change again.
	
	Denote 
	$\genrisk(h)=\inf_{\sigma}\probm^{\sigma}_{b_{h}}(\discpath{}{\discount,N-\len{h}}
	< \varthr -  \discpath{h_{-L}}{\discount})/\gamma^{\len{h}-N})$. We proceed 
	by backward induction on the depth of $h$. Clearly, for each leaf $h$ of 
	$\trsafe(\hroot)$ it holds $\genrisk=0$, so the statement holds. Now let 
	$h$ be any internal node of $\trsafe(\hroot)$. We have
\begin{align}
\genrisk(h)&=\min_{a\in \act}\Big(1-\sum_{o} \prob(h,hao) 
\cdot(1-\genrisk(hao))\Big) \nonumber\\  
&= \min_{a\in \act}\Big(1-\sum_{(h,hao)\in \trsafe(\hroot)} \prob(h,hao) 
\cdot(1-\genrisk(hao))\Big)\label{eq:1}\\
&= \min_{a\in \act}\Big(1-\sum_{(h,hao) \in \trexpl} \prob(h,hao) 
\cdot(1-\genrisk(hao))\Big)\label{eq:2}\\
&= \min_{a\in \act}\Big(1-\sum_{(h,hao) \in \trexpl} 
\prob(h,hao)(1-hao.U)\Big)\label{eq:3},
\end{align}

where individual equations are justified as follows: \eqref{eq:1} follows from 
the fact that $\prob(h,hao)=0 $ for each $a,o$ s.t. 
$(h,hao)\not\in\tree_{\pomdp}$ and $\genrisk(hao)=1$ for each $a,o$ s.t. $hao 
\not\in \trsafe(\hroot)$; \eqref{eq:2} follows from the fact that 
$\trsafe(\hroot)=\trexpl$; and \eqref{eq:3} follows from induction hypothesis. 
But during the call of procedure $\texttt{UpdateTrees}$ in which the last 
leaf-descendant of $h$ is added to $\trexpl$, the value of 
expression~\eqref{eq:3} is assigned to $h.U$ via computation on  
lines~\ref{algoline:dp1}--\ref{algoline:dp2}. This finishes the proof.

\end{proof}

\section{Proof of Lemma~\ref{lem:strat-convergence}.} 

We re-use some notation from the previous proof.

Due to existence of a feasible solution, it holds 
$\inf_{\sigma}\probm^{\sigma}_{\initd}(\discpath{}{\discount,N}
< \thr) \leq \rbound$.
Due to Lemma~\ref{lem:treeconv} it suffices to prove that $d_{\pi}$ converges 
to the optimal distribution whenever $\trexpl$ becomes equal to $\trsafe$. In 
such a case, the constrained MDP $\mdp$ admits a feasible solution. Let 
$\sigma$ be the optimal constrained-MDP policy in the MDP $\mdp$ obtained from 
$\treeclos$. Since states of this MDP are histories, $\sigma$ can be viewed as 
a  policy in the original POMDP $\pomdp$. Since in $\mdp$ the policy $\sigma$ 
satisfies the constraint given by $\constr$, in $\pomdp$ the policy ensures 
that a history of length $N$ belonging to $\trexpl$ is visited with probability 
at least $1-\rbound$. But as shown in the proof of 
Lemma~\ref{lem:search-convergence} (1.) these are exactly histories of length 
$N$ for which the payoff is at least $\thr$. Hence, in $\pomdp$, the policy 
$\sigma$ satisfies $\probm^{\sigma}_{\initd}(\discpath{}{\discount,N}
< \thr)\leq \rbound $. Conversely, any policy in $\pomdp$ satisfying the 
above constraint induces a policy in $\mdp$ satisfying 
$\E^{\pi}[\discpath{}{\discount}^\constr] \geq 1-\rbound$, where 
$\discpath{}{\discount}^\constr$ is the discounted sum of penalties.

It remains to prove that optimal constrained payoff achievable in $\mdp$ 
converges to the optimal risk-constrained payoff achievable in $\pomdp$. Since 
rewards in $\mdp$ are in correspondence with rewards in $\trexpl$, it suffices 
to show that for each leaf $h$ of $\trexpl$ with $\len{h}<N$ and for each 
action $a$ the attribute $h.V_a$ converges with probability 1 to the optimal 
expected payoff for horizon $N-\len{h}$ achievable in $\pomdp$ after playing 
action $a$ from belief $b_h$. But since the $V_a$ attributes are updated solely 
by POMCP simulations, this follows (for a suitable exploration constant) from 
properties of POMCP (Theorem~1 in~\cite{SV:POMCP}).

\section{Proof of Theorem~\ref{thm:convergence}.} 

We proceed by induction on $N$. For $N=0$ the statement is trivial. So assume 
that $N>1$. Since we assume the existence of a feasible solution to the EOPG 
problem, from Lemma~\ref{lem:search-convergence} (2.) it follows that with 
probability converging to one the $U$-attribute of the root of $\trexpl$ 
eventually becomes equal to 
$\inf_{\sigma}\probm^{\sigma}_{\initd}(\discpath{}{\discount,N}
< \varthr) \leq \rbound$ (the last inequality following from the existence of a 
feasible solution). Then, $\epsilon.U$ becomes $\rbound$ with probability 
converging to $1$, and when this happen the constrained MDP associated to 
$\treeclos$ has a feasible solution.

We first prove that the probability of RAMCP returning a payoff $<\thr$ is at 
most $\rbound$. Let $d_\pi$ be the distribution on actions returned by the 
first call of procedure \texttt{PlayAction}, and let $\{d^a\}_{a\in \act}$ be 
the corresponding set of risk distributions. From induction hypothesis and from 
the way in which the variables $\varthr$ and $\varrboundr$ are updated it 
follows that for each action $a$ and observation $o$, the probability that 
RAMCP launched from initial belief $b_ao$ returns payoff smaller than 
$\thr(a,o)= (\thr - \reward(a,o))/\gamma$ converges to a number $\leq d^a(o)$. 
The probability that the whole call of PAMCP returns payoff smaller than $\thr$ 
is 
\begin{align*}
&\sum_{a\in\act} d_\pi(a)\cdot\sum_{o} \prob(\epsilon,ao) \cdot d^a(o).
\end{align*}

But the above expression exactly expresses the probability that $C>0$ under 
$\pi$ in $\mdp$, and due to the construction of $\mdp$ such a probability 
equals $\E^{\pi}[\discpath{}{\discount}^\constr]$. Since $\pi$ is a feasible 
solution of the constrained MDP $\mdp$, we have 
$\E^{\pi}[\discpath{}{\discount}^\constr]\leq \epsilon.U $, where the last 
quantity becomes $\rbound$ with probability converging to $1$.

It remains to argue about convergence to optimality w.r.t. expectation. We 
again proceed by induction. For $N=0$ the statement is again trivial. From 
Lemma~\ref{lem:strat-convergence} we know that the distribution $d_\pi$ 
converges to the distribution used in the first step by some optimal solution 
$\sigma$ to the EOPG problem. For each $a\in \act$, $o\in \obs$, let 
$\sigma_{ao}$ be the fragment of this policy on the sub-tree rooted in $ao$. 
From the proof of Lemma~\ref{lem:strat-convergence} we get even stronger 
statement: if the explicit tree $\trexpl$ eventually becomes equal to $\trsafe$ 
(which happens with probability 1), the policy $\sigma$, to whose first step 
the algorithm converges, is an optimal solution to the associated constrained 
MDP $\mdp$. From this it follows that $\probm^{\sigma}_{b_{ao}} 
(\discpath{}{\discount,N-1}
< (\thr -  \reward(a,o))/\gamma)\leq d^a(o)$. So the EOPG problem with initial 
belief $b_{ao}$ with threshold $\thr_{ao} = \thr -  \reward(a,o))/\gamma$ and 
risk bound $d^{a}(o)$ has a feasible solution, and by the induction hypothesis 
we have that RAMCP on such a problem converges to the optimal risk-constrained 
expected payoff $v_{ao} = 
\E^{\sigma_{ao}}_{b_{ao}}[\discpath{}{\discount,N-1}]$. Hence, using RAMCP from 
belief $\initd$ results in convergence to expected payoff 
$\sum_{a\in\act}d^{\pi}(a)\cdot 
\sum_{{o}\in\obs}\prob(\epsilon,ao)\cdot(\rewards(\epsilon,ao) + \discount\cdot 
v_{ao} = \E^{\sigma}_{\initd}[\discpath{}{\discount,N}]$. Since $\sigma$ is 
optimal solution of the original instance of the EOPG problem, the result 
follows.

\section{Proof of Theorem~\ref{thm:safety}.}
 
During execution of RAMCP, $\trexpl$ can only grow in size. This allows us to 
prove the theorem by induction on $N$. For $N=0$ this is trivial. Let $N>0$. 
If, after the first search phase, $\epsilon.U>\rbound$, then we choose action 
$a$ minimizing the $U$-attribute. From induction hypothesis it follows that the 
probability that RAMCP return payoff smaller than $\thr$ is at most 
$\sum_{o\in\obs}\prob(\epsilon,ao)\cdot ao.U = \epsilon.U_a = \epsilon.U$. If 
$\epsilon.U\leq \rbound$ after the first search phase, then the probability 
that RAMCP return payoff smaller than $\thr$ is at most
\begin{align*}
&\sum_{a\in\act}d_{\pi}(a)\cdot \sum_{o\in\obs}\prob(\epsilon,ao)\cdot ao.U \\
&\leq \sum_{a\in\act}d_{\pi}(a)\cdot \sum_{o\in\obs}\prob(\epsilon,ao)\cdot 
d^a(o)\\
&\leq \rbound,
\end{align*}
where the first inequality follows from the fact that in each sub-tree of 
$\trexpl$ rooted in some history $h$, the probability mass of histories 
surpassing the threshold is $\leq 1-h.U$, and the second inequality follows 
from the fact that $\pi$ is a feasible solution to the constrained MDP $\mdp$.

\section{Relationship to Constrained POMDPs}

In principle, the EOPG problem can be encoded directly as a constrained-POMDP 
problem, by using indicator random variable for the event of surpassing the 
payoff threshold. However, this has several issues: 
\begin{itemize}
	\item
In~\cite{PMPKGB15:constrained-POMDP}, the incurred constraint penalties are 
discounted, while in the EOPG we would need the indicator variable to be 
undiscounted. 
\item 
Formulating EOPG problem as C-POMDP (with undiscounted constraints) would 
require extending states with the reward accumulated in the past. We could 
either a) discretize space of payoffs (which might result in large increase of 
state space if high precision is required), or b) consider only those payoffs 
accumulated on histories of length $<=N$ (horizon), or c) directly extend the 
state space with histories of length $<=N,$ i.e. explore the history tree of 
the original POMDP. Since b) entails analysing histories of length $<=N$ (in 
the worst case, all histories of length $<=N$ might yield different payoff), 
and c) allows us to formulate the problem as a C-MDP, we go in the latter 
direction. Using the C-MDP formulation forms only a part of our approach, there 
are other fundamental components such as sampling a promising sub-tree using 
MCTS. 
\end{itemize}
Also, our algorithm for EOPG is conceptually different from those used in the 
C-POMDP literature and provides different features:
\begin{itemize}
\item
In~\cite{UH10:constrained-pomdp-online}, they have an offline pre-processing 
step using PBVI. Our tool uses on-line Monte-Carlo tree search with linear 
programming. In general, simulation techniques are known to provide better 
scalability than point-based methods. Also, their algorithm is deterministic 
and thus produces deterministic policies, which are generally sub-optimal when 
compared to randomized policies. Finally, the paper makes no claims regarding 
the convergence of the algorithm to an optimal deterministic policy.
\item
In~\cite{PMPKGB15:constrained-POMDP} they use approximate linear programming 
to obtain approximate solutions to the C-POMDP problem. As they mention in the 
paper, their approach may yield policies that violate the C-POMDP constraints. 
Our algorithm for EOPG problem is such that if it finds a feasible solution, 
the risk bound is guaranteed (Theorem~\ref{thm:safety}).
\end{itemize}

\end{document}